\documentclass{article}

\usepackage[nonatbib, preprint]{neurips_2023}

\usepackage[utf8]{inputenc} 
\usepackage[T1]{fontenc}    
\usepackage{hyperref}       
\usepackage{url}            
\usepackage{booktabs}       
\usepackage{amsfonts}       
\usepackage{nicefrac}       
\usepackage{microtype}      
\usepackage{xcolor}         

\usepackage[numbers]{natbib} 
\usepackage{amsmath}
\usepackage{mathtools}
\usepackage{amsfonts,amssymb,amsthm}
\usepackage{multicol}
\usepackage{hyperref}
\usepackage[inline]{enumitem}
\usepackage{tikz}
\usepackage{pgfplots}
\usepackage{subcaption}

\definecolor{maroon}{rgb}{0.55,0,0}

\theoremstyle{plain}
\newtheorem{theorem}{Theorem}
\newtheorem{theorem*}{Theorem}

\newtheorem{proposition}{Proposition}
\newtheorem*{proposition*}{Proposition}

\theoremstyle{definition}

\theoremstyle{remark}

\title{Decorrelating neurons using persistence}

\author{%
  Rubén Ballester\\
  Departament de Matemàtiques i Informàtica\\
  Universitat de Barcelona\\
  Gran Via de les Corts Catalanes, 585, 08007 Barcelona, Catalonia, Spain \\
  \texttt{ruben.ballester@ub.edu} \\
  \And
  Carles Casacuberta \\
  Departament de Matemàtiques i Informàtica\\
  Universitat de Barcelona\\
  Gran Via de les Corts Catalanes, 585, 08007 Barcelona, Catalonia, Spain \\
  \texttt{carles.casacuberta@ub.edu} \\
  \And
  Sergio Escalera\\
  Departament de Matemàtiques i Informàtica\\
  Universitat de Barcelona\\
  Gran Via de les Corts Catalanes, 585, 08007 Barcelona, Catalonia, Spain \\
  \texttt{sergio.escalera@ub.edu} \\
}

\begin{document}

\maketitle

\begin{abstract}  
We propose a novel way to improve the generalisation capacity of deep learning models by reducing high correlations between neurons. 
For this, we present two regularisation terms computed from the weights of a minimum spanning tree of the clique whose vertices are the neurons of a given network (or a sample of those), where weights on edges are correlation dissimilarities.
We provide an extensive set of experiments to validate the effectiveness of our terms, showing that they outperform popular ones. Also, we demonstrate that naive minimisation of all correlations between neurons obtains lower accuracies than our regularisation terms, suggesting that redundancies play a significant role in artificial neural networks, as evidenced by some studies in neuroscience for real networks. We include a proof of differentiability of our regularisers, thus developing the first effective topological persistence-based regularisation terms that consider the whole set of neurons and that can be applied to a feedforward architecture in any deep learning task such as classification, data generation, or regression.
\end{abstract}

\section{Introduction}
Neural networks have proven to be powerful models to solve complex tasks. Usual neural networks show a high capacity to generalise properly beyond the training dataset used to fit their parameters~\citep{undestanding_requires_rethinking_generalisation}. Although there is no general explanation of why this happens yet, abundant literature is available to tackle this problem~\cite{in_search_of_robust_measures_of_generalisation, on_measuring_excess_capacity_in_neural_networks, pgdlcompetition, Jiang2020Fantastic, Kawaguchi_2022}. Moreover, many regularisation methods have been proposed to improve generalisation capacity from both theoretical and practical perspectives. According to experimental results, explicit regularisation may improve generalisation performance~\cite{undestanding_requires_rethinking_generalisation}. 

Evidence from neuroscience indicates that correlation among neurons is a significant factor in the brain's ability to encode and process information~\cite{cohenkohn2011,kohnsmith2005}. This was pointed out in~\cite{NEURIPS2020_f48c04ff}, where it was observed that the generalisation error of a deep network is monotonic with respect to the correlation between weight matrices of neurons or filters, suggesting that decreasing this correlation can be beneficial to improve the generalisation capacity of a network. Furthermore, in~\cite{cogswell2016reducing}, overfitting of neural networks was reduced by decorrelating their neuron activations. Overall, these studies suggest that a reduced correlation between neuron activations could improve the robustness of a network. 

However, neuroscience also suggests that redundancy appears naturally in brain circuits and is useful to perform certain computations~\cite{Hennig2018-ye, MIZUSAKI202174}. For this reason, aggressively minimising correlations between all activations or weights may be detrimental for the performance of a neural network.

In this work, we propose a way to minimise only the most relevant high correlations between neurons. 
For each batch of data during training, we compute regularisation terms based on edge weights of a minimum spanning tree of the clique generated by the most relevant neurons for the batch, based on an importance measure inspired by the activation criterion for neural network pruning presented in~\cite{molchanov2017pruning}. Edge weights in the clique are pairwise correlations between activation vectors of neurons.
In order to prove that our regularisation terms are almost everywhere differentiable,
we use the differential calculus framework for persistent homology developed in~\cite{Leygonie2022}. 

\subsection{Contributions}

The main contributions of our work can be described as follows:

\begin{enumerate}
  \item We propose two novel regularisation terms that minimise only some of the highest correlations of the most relevant neurons in a specific training batch. This approach allows for some redundancy in the neural network, unlike other articles such as~\cite{cogswell2016reducing, NEURIPS2020_f48c04ff}, which propose to minimise all pairwise correlations. 
  Each regularisation term employs a distinct method,
  and each of them outperforms the other in specific networks, thereby complementing each other.
  \item We use differentiable persistence descriptors to ensure differentiability of our regularisation terms, thus developing, to the best of our knowledge, the first topological regularisation terms that depend on the whole set of hidden internal representations of the neurons of a neural network.
  \item We provide an extensive set of experiments to validate the effectiveness of our topological regularisation terms. We also compare our regularisers with several popular regularisation terms and find that our regularisers achieve a better performance in the experiments.
  \end{enumerate}

The article is structured as follows. In Section~\ref{scn:related_work}, we analyse the current approaches to regularisation of neural networks using correlations and topological data analysis. In Section~\ref{scn:methodology}, we describe our topological regularisation terms and prove that they are differentiable almost everywhere. In Section~\ref{scn:results}, we describe the experiments performed to validate our regularisation terms and discuss their results. In Section~\ref{scn:limitations_future_work}, we discuss limitations of our approach and possible future work. Section~\ref{scn:conclusions} contains our conclusions, and basic facts about differentiability of persistence descriptors are detailed in the Appendix.

\section{Related work}\label{scn:related_work}
Many works in deep learning study how to apply explicit regularisation to improve the generalisation capacity of neural network models. Popular approaches include dropout~\cite{dropout_original_paper}, in which neurons of neural networks are dropped randomly during training, and the classical $l_1$ and $l_2$ regularisation terms~\cite{classical_l1_l2}, that control the size of  weights of a neural network. However, many regularisation approaches have been developed in the literature. Among them, some have used correlations between weights and activations of neurons to regularise neural networks. In particular, \cite{NEURIPS2020_f48c04ff}
uses weight correlations for convolutional and fully-connected layers to improve deep network learning outcomes, and \cite{cogswell2016reducing} proposes to minimise a loss function computed from a covariance matrix of neuron activations on a batch.
Both methods, however, have limitations: 
  \begin{enumerate}
    \item In~\cite{NEURIPS2020_f48c04ff}, the definition of weight correlation is hardcoded only for fully-connected and convolutional layers and thus does not apply to some relevant modern architectures. 
    \item In~\cite{cogswell2016reducing}, correlations between neurons are computed for a set of neurons defined by the user. This leaves the practitioner with the need to choose the neurons to decorrelate, which is not an easy task 
    for large neural networks.
    \item Both articles~\cite{cogswell2016reducing, NEURIPS2020_f48c04ff} reduce all 
    correlations at the same time, without taking into consideration that redundancy between neurons can be important. 
    Also, the regularisation terms proposed in these articles use only layerwise correlations, without taking into account interactions between neurons of different layers.
  \end{enumerate}

We propose two regularisation terms that, first, can be used in any feedforward architecture and are not restricted to correlations in the same layer, and secondly, they only decorrelate the neurons with highest correlations, thus allowing the neural network to be flexible enough to keep an amount of redundancy that might be useful for the task.

During the last years, the popularity of topological methods in machine learning has rapidly increased. An overall survey of these methods can be found 
in~\cite{survey_topological_machine_learning}. 
Particularly interesting for this paper is the use of topological priors for regularising neural networks. In~\cite{pmlr-v139-carriere21a, Leygonie2022}, 
frameworks for differential calculus on persistence barcodes were defined, allowing to optimise point cloud shapes and thus to construct topological regularisation terms. 
Thanks to the possibility of differentiating persistence diagrams, several neural network layers have been created with the objective of leveraging topological properties in the process of learning~\cite{pmlr-v108-carriere20a, pmlr-v108-gabrielsson20a, deep_learning_with_topological_signatures, learning_representations_of_persistence_barcodes, 10.5555/3495724.3497063, reinauer2022persformer}. In~\cite{a_topological_regularizer_for_classifiers_via_persistent_homology}, the first regularisation term for neural networks based on the topology of the decision region was proposed. From there, specific topological regularisation terms have been discussed for image segmentation~\cite{a_persistent_homology_based_topological_loss_function_for_multi_class_CNN_segmentation_of_cardiac_MRI,a_topological_loss_function_for_deep_learning_based_image_segmentation_using_persistent_homology, topology_preserving_deep_image_segmentation, hu2021topologyaware}, autoencoder latent space~\cite{pmlr-v97-hofer19a}, and classification using decision boundaries~\cite{pmlr-v89-chen19g}. 

Among the current approaches on regularising neural networks, the most similar to our method are the ones suggested in~\cite{intrisinc_dimension_persistent_homology_generalisation_neural_networks, topologically_densified_distributions}. In~\cite{topologically_densified_distributions}, zero-dimensional persistent homology is used to optimise the mass concentration of the internal representations in the last hidden layer assuming that the mini-batches used during training are equally distributed among all classes. In~\cite{intrisinc_dimension_persistent_homology_generalisation_neural_networks}, an upper bound of the generalisation gap of a neural network $\mathcal N$ was found theoretically in terms of the persistent homology dimension of the continuous set of weights generated during the training of $\mathcal N$. The higher the dimension, the higher the upper bound on the generalisation gap, and thus it was proposed to regularise neural networks by adding a topological term consisting of an approximation of the persistent homology dimension of the set of weights generated during training.

Although the existing topological regularisation approaches perform satisfactorily when used appropiately, most of them are restricted to specific tasks or have limitations that can be complemented with our approach, that is fundamentally different from previous methods, for the following reasons:
\begin{enumerate}
  \item Topological regularisation terms based on decision boundaries or the evolution of the whole set of weights do not consider the whole structure of a network, that provides substantial information about the network's generalisation~\cite{computing_the_testing_error_without_a_testing_set, what_does_it_mean_to_learn_in_deep_networks, Kawaguchi_2022, neural_persistence}.
  \item Topological regularisation terms designed for specific tasks, such as segmentation, or specific models, such as autoencoders, are too restrictive in general scenarios.
  \item In \cite{topologically_densified_distributions} it is assumed that the batches during training have equally distributed samples among the different labels and can be applied only to classification tasks. For unbalanced datasets or non-classification tasks, such assumptions need not be satisfied.
\end{enumerate}

Our topological regularisation terms based on correlations of network neurons are agnostic to the problem (regression, classification, generative models, etc.) and compatible with most model types (CNNs, transformer-based, MLPs, etc.), and they work \emph{on the whole structure of the network} by means of neuron activations.

\section{Methodology}\label{scn:methodology}

The presence of high correlations between neurons in a neural network may imply the existence of redundant features learned from the data. This becomes particularly significant when two neurons, denoted as $\nu_1$ and $\nu_2$, exhibit a Pearson correlation coefficient equal to~$\pm 1$.
In such cases, there exist real numbers $a\neq 0$ and $b$ such that $\nu_1 = a\nu_2 + b$ with probability one. Consequently, one of the neurons can be eliminated, resulting in a smaller neural network that represents the same function as the original neural network almost everywhere. This observation suggests that excessively high correlations between neurons restrict the network's capacity to fully utilise its expressivity for learning the problem, highlighting the need to prevent such behaviour.

\begin{figure}[t!]
  \centering
  \includegraphics[width=0.7\textwidth]{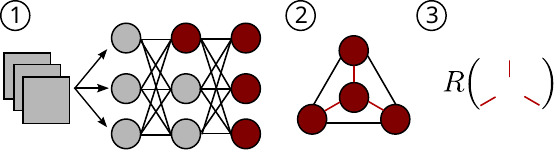}
  \caption{Sequence of steps needed to compute the proposed regularisation terms 
  $\mathcal{T}_1$ and $\mathcal{T}_2$.
  In the first step, the neural network is fed with the training batch and the activations for each neuron are computed. Then, a set of neurons per hidden layer
  (maroon nodes) is chosen. For the importance sampling algorithm described in the methodology, we take a fixed percentage of $0.5$\% of neurons per hidden layer
  and the whole set of neurons in the output layer. In the second step, the persistence diagram ---that is, a multiset of weights of the edges of a minimum spanning tree--- is computed (red edges).
  Finally, in the third step, a suitable real-valued function $\mathcal R$ from the weights of the minimum spanning tree is computed. By fixing batch examples, the computation of the real value given by $\mathcal{R}$ using the previous steps is differentiable almost everywhere with respect to neural network parameters due to the chain rule and Theorem~\ref{thm:differentiability_regularisation_terms}.}
  \label{fig:graph}
\end{figure}
However, we conjecture that fully avoiding correlated features may be detrimental for learning tasks. This is because
\begin{enumerate*}[label={(\arabic*)}]
  \item it imposes hard restrictions to the weights of neural networks during training, and
  \item there is evidence from neuroscience 
  suggesting that correlation is beneficial in brain
  operation~\cite{Hennig2018-ye, MIZUSAKI202174}.
\end{enumerate*}
In this paper, we propose to generate a \textit{balanced amount} of correlation between neurons by reducing only some of the highest ones. To do this, we use dimension zero persistent homology to build two regularisation terms that work in a complementary way. 

For a thorough introduction to persistent homology, we refer the reader to the book~\cite{edelsbrunner2022computational} and to the survey~\cite{survey_topological_machine_learning}. 
To a finite set $X$ and a symmetric function $d\colon X\times X\to\mathbb R_{\geq 0}$ such that $d(x,x)=0$ for all $x\in X$ one associates a \emph{persistence diagram} in every homological dimension greater than or equal to zero as described e.g.\ in~\cite{edelsbrunner2022computational} (details are also given in the Appendix). 

In this work we only use persistent homology in dimension zero. For this reason, since points in a zero-dimensional persistence diagram are aligned along the positive $y$ axis, we only focus on their $y$-coordinates. Hence we associate to each pair $(X,d)$ as above a finite multiset $D(X,d)$ of positive real numbers, ignoring the point at infinity. Such numbers are the weights of the edges of a minimum spanning tree of the undirected weigthed graph 
$(V,E,w)$ with $V=X$ and $E=\left\{\{u,v\}:u,v\in V\right\}$, with weights $w(\{u,v\})=d(u,v)$. This graph is a \emph{clique} (or a \emph{complete graph}) since every pair of vertices is joined by an edge, and a \emph{minimum spanning tree} (MST) is a subgraph without cycles containing all the vertices with the minimum possible total edge weight.

Given a zero-dimensional persistence diagram $D(X,d)$, we use the framework for differential calculus of~\cite{Leygonie2022} to build differentiable regularisation terms depending on $D(X,d)$. In particular, we are interested in maximizing the overall values of persistence diagrams for pairs $(X,d)$ where $X$ represents a set of neurons and $d$ is a dissimilarity function measuring correlations between them. 

Let us denote by
$\mathcal{N}\colon\mathcal X\to\mathcal Y$
a feedforward neural network, and let $\mathcal{D}=\left\{(x_1, y_1),\hdots, (x_n, y_n)\right\}\subseteq 
\mathcal X \times \mathcal Y$ 
be a dataset ---in our case, batches of the training dataset. Let $G_\mathcal{N}$ be the clique graph of the network~$\mathcal{N}$, whose vertices are the neurons of~$\mathcal{N}$.
As we do not know the marginal distribution of the data on~$\mathcal{X}$, we approximate the correlation between two neurons seen as random variables $u,v\colon\mathcal X\to\mathbb R$ using the sample correlation for the neuron activations in the dataset $\mathcal{D}$ restricted to the inputs. Recall that the \emph{sample correlation} of two vectors $x,y\in \mathbb R^n$ is defined as
\begin{equation*}\label{def:sample_correlation}
  \text{corr}(x,y)=\frac{\sum_{i=1}^n(x_i-\bar{x})(y_i-\bar{y})}{\sqrt{\sum_{i=1}^n (x_i - \bar{x})^2}\sqrt{\sum_{i=1}^n (y_i - \bar{y})^2}},
\end{equation*}
where $\bar{x}=\frac{1}{n}\sum_{i=1}^n x_i$ and $\bar{y}=\frac{1}{n}\sum_{i=1}^n y_i$. Thus, the sample correlation between two neurons $u$ and $v$ given $\mathcal{D}$ is $\text{corr}_\mathcal D(u,v) \triangleq \text{corr}\left(u(\mathcal{D}), v(\mathcal{D})\right)$ where $u(\mathcal{D})\triangleq(u(x_1),\hdots, u(x_n))$ and $v(\mathcal{D})\triangleq (v(x_1),\hdots, v(x_n))$.

In our case, neurons are considered to be similar when they share a large correlation in absolute value. Since correlations take values between $-1$ and $1$,
we define the \emph{correlation dissimilarity} between neurons as a function $d\colon V(G_\mathcal{N})\times V(G_\mathcal{N})\to [0,1]$ given by $d(u,v)=1-\left|\text{corr}_\mathcal D(u,v)\right|$. Given
this dissimilarity function, we can study 
correlations between any subset of neurons $V'\subseteq V(G_\mathcal{N})$ by means of the persistence diagram $D(V', d)$.  

By the cut property of minimum spanning trees, each MST of a clique contains, for each of its vertices, at least one edge with the minimum weight among its incident edges. In our case, this is translated into the fact that the diagram $D(V', d)$ contains 
the value 
\begin{equation*}
    1-\max_{u,v\in V',\,u\neq v}\left|\text{corr}_\mathcal D(u,v)\right|
\end{equation*}
among, possibly, 
other high correlations 
to form the MST. Therefore, by maximising the values of $D(V',d)$, we are in fact minimising a set of correlations between neurons in $V'$ containing the highest correlations achieved by neurons in the set.

Current neural networks contain an enormous quantity of neurons and computing a MST of the 
weighted clique
$(V(G_\mathcal{N}), d)$ is not feasible in many cases, as computing a minimum spanning tree has a complexity of $\mathcal{O}(e\cdot\alpha(e,v))$~\cite[Theorem~1.1]{mst_connected_graph}, where $\alpha(e, v)$ is the functional inverse of Ackermann's function~\cite{ackermann_function} and $e$ and $v$ are the number of edges and vertices, respectively, with $e=\binom{v}{2}$ because the graph is a clique. 

For this reason,
we consider, for each batch $\mathcal B=\{(x_1, y_1), \hdots, (x_{\left|\mathcal B\right|}, y_{\left|\mathcal B\right|})\}$ during training, a subset of neurons $V_\mathcal{B}\subseteq V(G_\mathcal{N})$ that may have 
smaller cardinality than $V(G_\mathcal{N})$. In particular, for the cases in which we cannot set $V_\mathcal{B}=V(G_\mathcal{N})$, we sample $V_\mathcal{B}$ using an \emph{importance sampling algorithm} for each batch. 
We take the top 
percentage $P$ of most important neurons of each layer given the batch, except for the last layer, where we take all the neurons, where $P$ is a hyperparameter depending on the size of the neural network. This is because the last layer showed to contain relevant information with respect to generalisation in other works like~\cite{topological_approaches_to_deep_learning}. For our experiments with big neural networks, we set $P$ to $0.5$\% due to practical hardware limitations. 

The \emph{importance} of a neuron given a batch is set to the average quantity of absolute activation achieved by the neuron, and it is inspired 
by the activation criterion for pruning presented in~\cite[Section~2.2]{molchanov2017pruning}. The higher this value for a neuron is, the more relevance we allot to the neuron. More precisely, 
denote $\bar{v}_\mathcal B =\left|\mathcal B\right|^{-1}\sum_{i=1}^{\left|\mathcal B\right|}\left|v(x_i)\right|$ and let $V(G_\mathcal{N})_l=\{v_1,\hdots, v_{n_l}\}$ be the neurons of the $l$-layer of $\mathcal{N}$ in any descending order of their 
$\bar{v}_\mathcal B$ values. 
In our case, we use the order given by the \texttt{argsort} function of TensorFlow. Let
\begin{equation*}
  V_\mathcal{B}=\bigcup_{l}V_{\mathcal{B}, l}
  \quad \text{where}
  V_{\mathcal B, l}=\begin{cases}
    \{v_1,\hdots, v_{\lfloor 0.005
    \, n_l\rfloor}\}\text{ if }l \neq L,\\[0.1cm]
    V(G_\mathcal{N})_l\text{ otherwise},
  \end{cases}
\end{equation*}
where $l$ iterates over all possible layers of $\mathcal N$ and $L$ is the number of the last layer.

Recall that maximising the values of $D(V', d)$ is equivalent to minimising a set of high correlations between neurons in~$V'$. By convention, we assume that regularisation terms are minimised by network training algorithms, so to maximise a function $f(\theta)$ we minimise its
opposite function $-f(\theta)$. We propose two regularisation terms that maximise persistence diagram values in different ways:
\begin{multicols}{2}
\noindent
\begin{equation}\label{eq:total_persistence}
\hspace{2.7cm}
\mathcal T_1(\theta) \triangleq -\! \!\!\!\!\!\sum_{y\in D(V_\mathcal B, d)}\!\!\!\!\!\!y, 
\end{equation}\columnbreak
\begin{equation}\label{eq:std_avg_regularisation_term}
\!\!\!\!\!\mathcal T_2(\theta) \triangleq -\alpha\bar{D}(V_\mathcal B, d)  + \beta\sigma(D(V_\mathcal B, d)),
\end{equation}
\end{multicols} 
\vspace{-0.5cm}
where $\alpha,\beta\in\mathbb R_{\geq 0}$ are weight parameters, $\theta$ is the set of parameters of the neural network being trained, and 
\begin{equation*}
    \bar{D}(V_\mathcal B, d)=\frac{1}{\left|D(V_\mathcal B, d)\right|}\sum_{y\in D(V_\mathcal B, d)}\!\!\!\!\!\!y,\qquad
    \sigma^2(D(V_\mathcal B, d))=\frac{1}{\left|D(V_\mathcal B, d)\right|}\sum_{y\in D(V_\mathcal B, d)}\!\!\!\!\!(y-\bar{D}(V_\mathcal B, d))^2
  \end{equation*}
are the mean and variance of the total persistence of $D(V_\mathcal B, d)$.
The regularisation term $\mathcal{T}_1$ given by Equation~\eqref{eq:total_persistence} maximises the sum of 
values in the persistence diagram, while the regularisation term $\mathcal{T}_2$ given by Equation~\eqref{eq:std_avg_regularisation_term} is a more involved term that focuses on how the entries of the persistence diagram are distributed, minimising their dispersion and maximising their average value.
In our case, we pick
$\alpha=\beta=1/2$ 
since we treat mean and dispersion with the same strength. Further exploration and optimisation of these hyperparameters is left for future work.

\newpage

\begin{theorem}\label{thm:differentiability_regularisation_terms}
Let $c,n\geq 2$ and let $d\colon\mathbb R^{n}\times\mathbb R^{n}\to[0,1]$ be 
given by $d(x,y)=1-\left|\textnormal{corr}(x,y)\right|$
where
$\textnormal{corr}$ 
denotes correlation.
There exists an open dense subset $\mathfrak D_{c,n}\subseteq \mathbb R^{cn}$ such that
the functions
\[      \mathcal{T}_1(x_1,\hdots,x_c) \triangleq -\!\!\!\!\!\!\sum_{y\in D(X,d)}\!\!\!\!\!y,
\qquad  \mathcal{T}_2(x_1,\hdots,x_c)  \triangleq -\alpha\bar{D}(X, d)  + \beta\sigma(D(X, d))
\]
are $\mathcal{C}^\infty$ on $\mathfrak D_{c,n}$ for all $\alpha,\beta\in\mathbb{R}_{\ge 0}$,
where $X=\{x_1,\hdots,x_c\}$ and $\bar{D}(X, d)$ and $\sigma(D(X, d))$ 
denote average and standard deviation, respectively, of the zero-dimensional persistence diagram $D(X, d)$.
\end{theorem}
A proof of this result is provided in the Appendix.

Using the chain rule, our regularisation terms are well defined as soon as the neuron activations of the set of neurons $V_\mathcal B=\{\nu_1,\hdots,\nu_c\}$ in the batch $\mathcal B=\left\{x_1,\hdots,x_n\right\}$ form a vector 
\[
\mathbf{v}=(\nu_1(x_1),\hdots,\nu_1(x_n)),\hdots,(\nu_c(x_1),\hdots, \nu_c(x_n))\in\mathbb R^{cn}
\]
such that $\mathbf{v}\in\mathfrak D_{c,n}$ and such that the neuron activations are obtained in a differentiable way. Experimentally, we need not control when this vector $\mathbf{v}$ is inside $\mathfrak D_{c,n}$ thanks to the fact that $\mathfrak D_{c,n}$ is a dense set.
However, we note that ignoring points where non-differentiability may occur in the domain could introduce errors in some iterations during training, as it may also happen with ReLU~\cite{relu_derivative_zero}.

\section{Results}\label{scn:results}
In this section, we first describe the experimental setup and the computational resources that we use to validate the hypothesis stated in the previous section. This is done in Subsection~\ref{scn:experiments}. Then, we present and discuss the results in Subsection~\ref{scn:results_analysis}. The code used to perform these experiments is attached as supplementary material \footnote{\url{https://github.com/rballeba/DecorrelatingNeuronsUsingPersistence}}.

\subsection{Experimental setup}\label{scn:experiments}
We prove the plausibility of the hypotheses that we formulated in Section~\ref{scn:methodology} by making two blocks of experiments in which we train several neural networks with different regularisation terms, including our proposed ones, and without regularisation terms. For each block, we train several models following a common architecture: multilayer perceptron models for the first block and VGG-like models for the second one. The networks of the first block are trained in the MNIST dataset whereas the networks of the second one are trained in CIFAR-10. In both blocks, we explore the same set of weights for the regularisation terms. Finally, to compare the accuracies of our proposed regularisation terms to the other alternatives, we use the Friedman statistical test with its Nemenyi post-hoc. Details of the experiments, training procedures, and evaluation methods are provided through this section.

\medskip

\textbf{Multilayer perceptron experiments}
\nopagebreak

In the first block of experiments, we examine our regularisation terms in a simplified problem. We train three different multilayer perceptron architectures with $1000$ hidden neurons, labelled \texttt{0}, \texttt{1}, and \texttt{2} using the MNIST dataset~\cite{lecun2010mnist}. Networks \texttt{0} and \texttt{1} share the same fully connected architecture. However, network \texttt{1} is trained using dropout with a $50$\% probability of dropping a hidden neuron at each iteration. Specifically, architectures \texttt{0} and \texttt{2} have a trapezium shape consisting of a sequence of hidden layers of $450$, $350$, and $200$ neurons for the first network, and of $300$, $250$, $200$, $150$, and $100$ for the second one, respectively.

\medskip

\textbf{PGDL experiments}
\nopagebreak

In the second block of experiments, the objective is to see if the method scales properly to more complex datasets and models. We train eight different neural network architectures from the PGDL dataset introduced in~\cite{pgdlcompetition} during NeurIPS 2020 competition track. The PGDL dataset is a collection of tasks where each task is composed by one dataset and a set of different neural network architectures trained with the dataset of the task.  The eight different neural network architectures we take belong to the first task, that is composed of VGG-like neural networks and the CIFAR10 dataset~\cite{krizhevsky2009learning}. The architectures we selected are the ones corresponding to the numbers \texttt{20}, \texttt{21}, \texttt{22}, \texttt{23}, \texttt{148}, \texttt{149}, \texttt{150}, and \texttt{151} from the dataset. Architectures \texttt{20}, \texttt{21}, \texttt{148}, and \texttt{149} are the same as the architectures \texttt{22}, \texttt{23}, \texttt{150}, and \texttt{151}, but with a layerwise dropout probability of $0.5$, respectively. The difference between models \texttt{22} and \texttt{23} is the width of their convolutions, where architecture \texttt{22} has convolution widths of $256$  and architecture \texttt{23} has convolution widths of $512$. Finally, the architectures \texttt{150} and \texttt{151} are the same as the architectures \texttt{22} and \texttt{23} but with one more dense layer.

\medskip

\textbf{Training procedures}
\nopagebreak

In these experiments, we train the different architectures with different regularisation terms weighted with several values. To train the neural networks in the different experiments, we replicate approximately the training performed by the PGDL dataset used in the second block of experiments.

For both blocks of experiments we split the data into training, validation, and test datasets. For the MNIST dataset, we split the original training dataset into new training and validation datasets with $80$\% and $20$\% of the original data, respectively. Finally, we use the original test dataset as test dataset. For the CIFAR10 dataset, we split the original training dataset into new training and validation datasets, where we choose 
$1000$ examples of each class randomly for the validation dataset and we place the remaining examples into the training dataset. Again, we reuse the original test dataset.

For the training procedure, we train for a maximum of $1200$ epochs with early stopping after $20$ epochs of no improvement, and with a batch size of $256$. The algorithm used for training is the usual stochastic gradient descent (SGD) with momentum $0.9$. For the first block of experiments, we use an adaptive learning rate $\alpha_i\triangleq \alpha_{0}\cdot (0.95)^{i/3520}$ where $i$ is the iteration where the learning rate $\alpha_i$ is used and $\alpha_0=0.01$. For the second block of experiments, we use a fixed learning rate of $0.001$, for which we obtained similar accuracies to the ones given by the original trainings of the PGDL neural networks. 

Let $\text{CCE}(\theta)$ denote the categorical cross entropy loss for a fixed neural network, clear from the context, with a set of parameters $\theta$. For each of the networks described before we perform several trainings with the different regularisation terms that we study, weighed by different values. In particular, each training minimises a loss function 
\begin{equation}\label{eq:loss_function_pgdl}
  \mathcal{L}(\theta) = \text{CCE}(\theta) + \omega \mathcal{R}(\theta),
\end{equation}
where $\omega\in\{10^{-6}, 10^{-5}, 10^{-4}, 0.001, 0.01, 0.1, 1. 5, 10, 100\}$ represents one of the possible weight values used in the experiments, and $\mathcal R(\theta)$ represents one of the regularisation terms. We also train the models without any regularisation term, i.e., with $\omega=0$. 

We use the full and sampled version of our regularisers $\mathcal{T}_1$ and $\mathcal{T}_2$ for the first and second blocks of experiments, respectively, due to the small size of the networks of the first block and to the large size of the networks of the second one. To see if reducing only some correlations between neurons is better than minimising all of them, we also study the regularisation term $\mathcal{C}(\theta)$ defined as
\begin{equation}\label{eq:term_minimising_everything}
  \mathcal{C}(\theta) =\frac{1}{|\mathfrak C|}\sum_{(x,y)\in\mathfrak C}|\text{corr}_{\mathcal{B}}(x,y)|,
\end{equation}
where $\mathfrak C = \left\{(x,y)\in V_\mathcal{B}\times V_\mathcal{B}: x\neq y \text{ and }\text{corr}_{\mathcal{B}}(x,y)\neq 0\right\}$, and $V_\mathcal{B}$ is defined as for the terms $\mathcal{T}_1$ and $\mathcal{T}_2$ in Section~\ref{scn:methodology}. Note that for the first block of experiments we consider all the non-input neurons and for the second block of experiments we perform the same sample of neurons due to the computational complexity of computing all the possible pairwise correlations for each iteration of the training. Finally, we also train the networks with the classic $l_1$ and $l_2$ regularisation terms \cite{classical_l1_l2}.

\medskip

\textbf{Evaluation procedure}
\nopagebreak

To evaluate the performances of the regularisation terms compared, we use a Friedman test with the Nemenyi post-hoc, as proposed in~\cite{significance_test}, and we report the test accuracies for each regularisation term and network. To obtain test accuracies, we choose, for each term and network, the weight that maximises the validation accuracy after training. Then, we compute the test accuracy using the selected weight. For the training procedures without regularisation terms, we compute directly their test accuracies.

\subsubsection{Resources used and computation}
The experiments were computed in a server with 503~GB of RAM, a CPU AMD EPYC 7452 32-Core Processor with a frequency up to 3.35~GHz, and seven GPUs NVIDIA GeForce RTX 3090 with 24~GiB of memory. The storage consisted of 3 Samsung SSDs, two of them with 3840~GB of memory and the other one with 960~GB. All the experiments were executed in parallel using one of the GPUs per experiment. The computational bottlenecks were related to the computation of correlation matrices of neurons and persistence diagrams. The first process was done using TensorFlow (in GPU mode) and the second one was performed using the library \texttt{giotto-ph}~\cite{burella2021giottoph}.

\subsection{Results and analysis}\label{scn:results_analysis}
\begin{table}
  \caption{Nemenyi $p$-value matrix. Each row  
  represents a different regularisation term except for the first row, 
  that represents 
  training without regulariser. Each cell is the $p$-value of the Nemenyi post-hoc test. The null-hypothesis of this test is that there is no difference between the accuracies yielded by the two training approaches. Those
  $p$-values lower than or equal to $0.05$ are bolded. $\mathcal{T}_1$:~First
  topological regularisation term as in~\eqref{eq:total_persistence};
  $\mathcal{T}_2$:~Second 
  topological regularisation term, as in~\eqref{eq:std_avg_regularisation_term}; 
  $l_1$:~Lasso regression regularisation term; $l_2$:~Ridge regression regularisation term; $\mathcal{C}$: Regularisation term minimising all the pairwise correlations in $V_\mathcal{B}$, as in~\eqref{eq:term_minimising_everything}; 
  $\emptyset$:~no regulariser.}
  \label{table:nemenyi_pgdl}
  \centering
\begin{tabular}{ccccccc}
\toprule
          & $\emptyset$ & $\mathcal{T}_1$ & $\mathcal{T}_2$ & $l_1$ & $l_2$ & $\mathcal{C}$\\
\midrule
 $\emptyset$   &       & \textbf{0.018} & \textbf{0.002} & 0.900       &   0.785 & 0.900        \\
 $\mathcal{T}_1$      &      &      & 0.900     & \textbf{0.036}      &   0.380   & 0.159    \\
 $\mathcal{T}_2$  & &  &       & \textbf{0.006} &   0.122 & \textbf{0.036}  \\
 $l_1$    &        &   &  &       &   0.900 & 0.900        \\
 $l_2$    &      &    &      &      &      & 0.900       \\
\bottomrule
\end{tabular}
\end{table}

\begin{table}
  \caption{Test accuracies for the different training procedures and networks. Each row represents a training with a different topological regularisation term except for the first row, that represents a training without a regularisation term. Each column represents a network in the PGDL dataset. Each cell is the test accuracy of the regularisation term with the weight that obtained the best validation accuracy for the term in the network specified by the column. Best accuracies per network are bolded. Rows $\mathcal{T}_1$, $\mathcal{T}_2$, $l_1$, $l_2$, $\mathcal{C}$, and $\emptyset$ are described in Table~\ref{table:nemenyi_pgdl}.}
\label{table:test_accuracies_pgdl}
  \centering
  \begin{tabular}{cccccccccccc}
    \toprule
    &    \multicolumn{3}{c}{MNIST and MLP} & \multicolumn{8}{c}{PGDL and VGG-like}\\
      \cmidrule(r){2-4}  \cmidrule(l){5-12}
           & 0 & 1 & 2   &     20 &     21 &     22 &     23 &    148 &    149 &    150 &    151 \\
    \midrule
     $\emptyset$ & \textbf{0.929} & 0.501  & 0.636 & 0.681 & 0.680 & 0.685 & 0.682 & 0.672 & 0.677 & 0.675 & 0.680 \\
     $\mathcal T_1$  &0.928 & \textbf{0.547} & \textbf{0.883}    & 0.687  & \textbf{0.705}  & 0.675  & 
     0.700    & 0.688  & \textbf{0.704}  & 0.678  & \textbf{0.698}  \\
     $\mathcal T_2$ & 0.923 & 0.540 &   0.879  & \textbf{0.691}  & 0.701  & \textbf{0.688}  & \textbf{0.706}  & \textbf{0.689}  & 0.698  & \textbf{0.688}  & 0.695  \\
     $l_1$ & 0.914 & 0.536  & 0.870   & 0.682  & 0.680   & 0.682  & 0.683  & 0.677  & 0.675  & 0.685  & 0.678  \\
     $l_2$ & 0.919 & 0.531 & 0.878  & 0.681  & 0.688  & 0.686  & 0.683  & 0.680   & 0.680   & 0.681  & 0.679  \\
     $\mathcal{C}$ & 0.923 & 0.530 & 0.881 & 0.679  & 0.687  & 0.680  & 0.686  & 0.678  & 0.69  & 0.683  & 0.674  \\
    \bottomrule
    \end{tabular}
\end{table}

Table~\ref{table:test_accuracies_pgdl} contains test accuracies for each network in the MNIST and PGDL experiments for the different regularisation terms and the training procedure without regularisation term. For each regularisation term and network, we compute the test accuracy for the weight of the regularisation term that obtained the best validation accuracy for the specific network. Using these accuracies, the Friedman test with null hypothesis that all the training algorithms
are equivalent~\cite{significance_test} 
gives a $p$-value of $1.04e-05 < 0.05$, so we reject the null hypothesis. Therefore, we perform the Nemenyi post-hoc test, obtaining the $p$-value matrix shown in Table~\ref{table:nemenyi_pgdl}.

When comparing test accuracies individually, $\mathcal{T}_1$ and $\mathcal{T}_2$ outperform the other training methods for all the networks except for the model \texttt{0}.  Concretely, $\mathcal{T}_1$ and $\mathcal{T}_2$ obtain the best test accuracies in five of the networks each, respectively, out of eleven total networks. Due to this, there is no winner between the two proposed regularisation terms. This is contrasted by the Nemenyi matrix, where the $p$-value between both methods has a high value of $0.900$. Also, both of them are significantly better than training without regularisation term, according to the Nemenyi $p$-value matrix, obtaining $p$-values of $0.018$ and $0.002$ for $\mathcal{T}_1$ and $\mathcal{T}_2$, respectively.

Regarding the differences between minimising all the pairwise correlations in the set of relevant neurons and minimising only the highest ones, i.e., the differences between the regularisation terms $\mathcal{T}_1$, $\mathcal{T}_2$, and $\mathcal{C}$, we see that both regularisation terms obtain low $p$-values in the Nemenyi test. In particular, for $\mathcal{T}_2$ and $\mathcal{C}$ we obtain a significant low $p$-value of $0.036$, that validates the hypothesis that minimising only the highest correlations is better than minimising all the correlations for a sampled set of relevant neurons in our experiments. 

Concerning the classical regularisation terms, both regularisation terms $\mathcal{T}_1$ and $\mathcal{T}_2$ obtain $p$-values lower than or equal to $0.05$ with respect to $l_1$, making our regularisation terms significantly better than $l_1$. As for $l_2$, the $p$-values of $\mathcal{T}_1$ and $\mathcal{T}_2$ with respect to $l_2$ are much lower than the other $p$-values for $l_2$, although not as low as for $l_1$. This, together to the fact that $l_2$ does not obtain the best accuracy for any network, suggests that our regularisation terms are probably better than $l_2$, although more experiments are needed to confirm this claim.

Overall, we see a trend in the results that $\mathcal{T}_1$ and $\mathcal{T}_2$ are superior to the other training approaches, which supports the hypothesis that we stated in the Methodology.

\section{Limitations and future work}\label{scn:limitations_future_work}
The computational cost of determining persistence diagrams, which is a crucial component of the pipeline to compute our proposed regularisation terms, can pose limitations in practical scenarios that require fast training or involve huge networks. To extend the applicability of our regularisation terms, an efficient algorithm for computing persistence diagrams must be developed. One approach is to leverage the computational power of GPUs, as done in the \texttt{Ripser++} package~\cite{zhang2020gpu}. However, \texttt{Ripser++} does not generate all the necessary information required for computing gradients in differentiable persistent homology, thus necessitating further development. Moreover, dimension zero persistent homology represents a unique case wherein persistence diagrams can be computed using any algorithm capable of computing minimum spanning trees of weighted graphs. Therefore, advancements in implementing these algorithms efficiently in a distributed or GPU-accelerated manner, as discussed in~\cite{gpu_mst} or~\cite{sanders2023engineering}, could be of considerable interest.

Other crucial bottleneck in our pipeline is the computation of pairwise correlations. With a large number of neurons, it is impractical to compute correlations for all neurons at each training step due to constraints in execution time and memory. We hypothesise that neuron selection greatly impacts the performance of regularisation terms, and thus, further study is needed. For large neural networks, a potential solution could involve computing regularisation terms within different clusters of neurons grouped by the proximity of their layers, as we hypothesise that \textit{distant} neurons have low correlation due to non-linear activations applied from one layer to the other.

The last limitation is related to the theoretical guarantees of our method. All our process is experimental, and we cannot guarantee that these regularisers work properly in any kind of learning task, despite neuroscience heuristic and experimental evidences. Further work to understand how correlations affect the generalisation capacity of neural networks is needed.
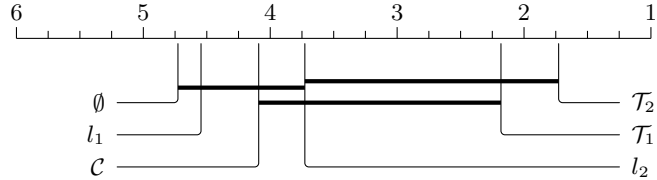
\begin{figure}[t!]
\centering
\begin{tikzpicture}[
  treatment line/.style={rounded corners=1.5pt, line cap=round, shorten >=1pt},
  treatment label/.style={font=\small},
  group line/.style={ultra thick},
]

\begin{axis}[
  clip={false},
  axis x line={center},
  axis y line={none},
  axis line style={-},
  xmin={1},
  ymax={0},
  scale only axis={true},
  width={\axisdefaultwidth},
  ticklabel style={anchor=south, yshift=1.3*\pgfkeysvalueof{/pgfplots/major tick length}, font=\small},
  every tick/.style={draw=black},
  major tick style={yshift=.5*\pgfkeysvalueof{/pgfplots/major tick length}},
  minor tick style={yshift=.5*\pgfkeysvalueof{/pgfplots/minor tick length}},
  title style={yshift=\baselineskip},
  xmax={6},
  ymin={-4.5},
  height={5\baselineskip},
  xtick={1,2,3,4,5,6},
  minor x tick num={3},
  x dir={reverse}
]

\draw[treatment line] ([yshift=-2pt] axis cs:1.7272727272727273, 0) |- (axis cs:1.2272727272727273, -2.0)
  node[treatment label, anchor=west] {$\mathcal{T}_2$};
\draw[treatment line] ([yshift=-2pt] axis cs:2.1818181818181817, 0) |- (axis cs:1.2272727272727273, -3.0)
  node[treatment label, anchor=west] {$\mathcal{T}_1$};
\draw[treatment line] ([yshift=-2pt] axis cs:3.727272727272727, 0) |- (axis cs:1.2272727272727273, -4.0)
  node[treatment label, anchor=west] {$l_2$};
\draw[treatment line] ([yshift=-2pt] axis cs:4.090909090909091, 0) |- (axis cs:5.2272727272727275, -4.0)
  node[treatment label, anchor=east] {$\mathcal{C}$};
\draw[treatment line] ([yshift=-2pt] axis cs:4.545454545454546, 0) |- (axis cs:5.2272727272727275, -3.0)
  node[treatment label, anchor=east] {$l_1$};
\draw[treatment line] ([yshift=-2pt] axis cs:4.7272727272727275, 0) |- (axis cs:5.2272727272727275, -2.0)
  node[treatment label, anchor=east] {$\emptyset$};
\draw[group line] (axis cs:1.7272727272727273, -1.3333333333333333) -- (axis cs:3.727272727272727, -1.3333333333333333);
\draw[group line] (axis cs:3.727272727272727, -1.5333333333333332) -- (axis cs:4.7272727272727275, -1.5333333333333332);
\draw[group line] (axis cs:2.1818181818181817, -2.0) -- (axis cs:4.090909090909091, -2.0);

\end{axis}
\end{tikzpicture}
  
  \caption{Critical difference diagrams~\cite{significance_test} for the Friedman and Nemenyi post-hoc statistical tests conducted in both blocks of experiments. The position of each training approach on the diagram corresponds to its average rank based on the test accuracies of the trained network ---see Table~\ref{table:test_accuracies_pgdl}. Lower ranks indicate that a regularisation term, or the training without regulariser, outperforms competitors with higher ranks. Regularisation terms are connected if the $p$-value obtained from the Nemenyi post-hoc test is greater than $0.05$. See Table~\ref{table:nemenyi_pgdl} for a description of the training approaches and the $p$-values.}
  \label{fig:cd_diag}
\end{figure}
\section{Conclusions}\label{scn:conclusions}
In this work, we introduced regularisation terms that minimise high correlations between the most important neurons given a training batch, by 
maximising the values of their zero-dimensional persistence diagram computed with the dissimilarity function $d(u,v)=1-\left|\text{corr}_\mathcal D(u,v)\right|$. Our regularisation terms outperformed classical regularisation terms and significantly improved the performance compared to minimising all pairwise correlations of 
important neurons in the MNIST and CIFAR10 datasets with MLP and VGG-like architectures, respectively. Additionally, we demonstrated that, when minimising 
higher correlations using persistent homology, several loss functions that are used with the same objective can yield different performances, suggesting that, for
differentiable persistent descriptors,
the choice of a loss function is a crucial step in the process. These findings support the hypothesis that neuron correlations play a crucial role in the generalisation capacity of neural networks, consistent with previous studies such as~\cite{cogswell2016reducing, NEURIPS2020_f48c04ff}. 

Our results also show that topological regularisation terms can be used to improve the performance of neural networks not only by considering the final representations of the data, but also by looking at the intermediate representations as well. This, along with other topological methods in deep learning, highlights the relevance of topological data analysis tools in understanding the behaviour of neural networks and the importance of
shape features of the data and the models on the generalisation capacity of neural networks. We hope that our work can encourage further research in this
direction, as well as the development of new topological methods for deep learning. In summary, our findings provide valuable insights into the role of topology and neuron correlations in deep learning and their potential for future
advances in the field.

\section*{Acknowledgements}
This work was supported by the Ministry of Science and Innovation of Spain through the research projects PID2022-136436NB-I00 and PID2020-117971GB-C22, and the Ministry of Universities of Spain through contract FPU21/00968.

{\small
\bibliographystyle{plainnat}
\bibliography{main}
}

\appendix

\section{Differentiability of functions on persistence diagrams}

In this appendix, we prove
Theorem~\ref{thm:differentiability_regularisation_terms} using methods and results from~\cite{Leygonie2022}.
We consider finite ordered sets $X\subseteq \mathbb R^n$ 
with $c$ elements, where $n\ge 2$ and $c\ge 2$.
Each such set $X$ corresponds to an element $(p_1,\dots,p_c)\in\mathbb{R}^{cn}$, where $p_i\in\mathbb{R}^n$ denotes the $i$th point of~$X$.

Given points $p_1,\dots,p_c$ in $\mathbb{R}^n$, where $c\ge 2$, recall that the \emph{covariance} between $p_i$ and $p_j$ is
\begin{equation*}
    \text{cov}(p_i, p_j)=\frac{1}{n}\sum_{k=1}^n(p_{i,k}-\bar{p}_i)(p_{j,k}-\bar{p}_j),
\end{equation*}
where $\bar{p_i}=\frac{1}{n}\sum_{k=1}^n p_{i,k}$.
Since the covariance function is polynomial on the entries, the set
\[
\mathcal{D}_{c,n} =\left\{(p_1,\hdots,p_c)\in\mathbb R^{cn}:
\text{$\text{cov}(p_i, p_j)\neq 0$ for all $i,j\in\{1,\hdots, c\}$}\right\}
\]
is open and dense in $\mathbb{R}^{cn}$.

Suppose given a continuous function $d\colon\mathbb R^n\times\mathbb R^n\to \mathbb R_{\geq 0}$ such that $d(x,y)=d(y,x)$ for all $x$ and~$y$ and $d(x,x)=0$ for every~$x$ ---we call such a function a \emph{dissimilarity}.
The \emph{persistence diagram} of $X=\{p_1,\dots,p_c\}$ in homological dimension $k$ with respect to the given dissimilarity $d$ is an unordered set of points $\{(b_i,d_i)\}_{i\in I}$ in $\mathbb{R}^2$ where $b_i$ is the birth parameter and $d_i$ is the death parameter (possibly infinite) of an element in a full set $I$ of linearly independent generators of $k$th simplicial homology $H_k(V_t(X,d))$ of the \emph{Vietoris--Rips filtered simplicial complex} $\{V_t(X,d)\}_{t\ge 0}$. This simplicial complex has an $m$-simplex at level $t$ for every collection of points $p_{i_0},\dots,p_{i_m}$ in $X$ such that $d(p_{i_r},p_{i_s})\le t$ for $i_r,i_s\in\{0,\dots,m)$. By convention, persistence diagrams include all points in the diagonal $\Delta^{\infty}=\{(x,x):x\ge 0\}$ with infinite multiplicity. Simplicial homology is computed with coefficients in any field.

A persistence diagram in homological dimension $k$ can be viewed as a function $B_k={\rm Dgm}_k\circ F$ defined on $\mathbb{R}^{cn}$ taking values in the set $\mathbf{Diag}$ of families of points with multiplicities in the upper quadrant of~$\mathbb{R}^2$ extended with points at infinity: \begin{equation}
\label{eqn:generating_persistence_diagram}
\mathbb R^{cn} \xrightarrow[]{\hspace{0.4cm}F\hspace{0.4cm}}
\mathbb R^K\xrightarrow[]{\hspace{0.2cm}\text{Dgm}_k\hspace{0.2cm}}\mathbf{Diag}.
\end{equation}
Here we denote by $\mathbb{R}^K$ the set of all functions $f\colon K\to\mathbb{R}$, where $K$ is the collection of nonempty subsets of $\{1,\dots, c\}$, which we view as faces of a $(c-1)$-dimensional simplex. The function $F$ is defined as 
\[
F(X)(\sigma)=\max_{i,j\in\sigma}\,d(p_i, p_j)
\]
for $\sigma\in K$ and $X=\{p_1,\dots,p_c\}$. The function $\text{Dgm}_k$ assigns to each  function $f\colon K\to\mathbb{R}$ the corresponding Vietoris--Rips persistence diagram in homological dimension~$k$, where $f$ is treated as a filtering function on the faces of a $(c-1)$-dimensional simplex.

Differentiability of functions valued in $\mathbf{Diag}$ is defined in \cite{Leygonie2022} as follows. 
For $m,\ell\in\mathbb{Z}_{\ge 0}$, consider the quotient map $Q_{m,\ell}\colon\mathbb R^{2m}\times\mathbb R^\ell\to\mathbf{Diag}$ sending each point
\[
\tilde{D}=(x_1,y_1,\hdots,x_m, y_m, z_1,\hdots, z_\ell)\in\mathbb R^{2m}\times\mathbb R^\ell
\]
to the diagram obtained by forgetting the order of the points:
\begin{equation*}
    Q_{m,p}(\tilde{D})=\left\{(x_i, y_i)\right\}_{i=1}^m\cup\left\{(z_j,\infty)\right\}_{j=1}^\ell\cup\Delta^\infty.
  \end{equation*}

Let $\mathcal{M}$ be a smooth manifold and let
$B\colon \mathcal{M}\to\mathbf{Diag}$ be any map. For $x\in\mathcal{M}$ and $r\in\mathbb{Z}_{\geq 0}\cup\{\infty\}$, the map $B$ is said to be \emph{$r$-differentiable} at $x$ if there exists an open neighborhood $U$ of $x$ and there exist integers $m,\ell\in\mathbb{Z}_{\geq 0}$ and a map $\tilde{B}\colon U\to\mathbb R^{2m}\times\mathbb R^\ell$ of class $\mathcal{C}^r$ such that $B=Q_{m,\ell}\circ \tilde{B}$ on~$U$.
Similarly, for a smooth manifold $\mathcal{N}$, a map 
$V\colon\mathbf{Diag}\to\mathcal{N}$ is said to be \emph{$r$-differentiable} at a diagram~$D$, where $r\in\mathbb{Z}_{\geq 0}\cup\{\infty\}$, if for all $m,\ell\in\mathbb Z_{\geq 0}$ and all $\tilde{D}\in\mathbb{R}^{2m}\times\mathbb{R}^{\ell}$ such that $Q_{m,\ell}(\tilde{D})=D$ the map $V\circ Q_{m,\ell}\colon \mathbb R^{2m}\times\mathbb R^\ell\to\mathcal N$ is $\mathcal{C}^r$ on an open neighborhood of $\tilde{D}$.

As proved in \cite[Proposition~3.14]{Leygonie2022},
if a function $B\colon\mathcal M\to\mathbf{Diag}$ is $r$-differentiable at $x\in\mathcal{M}$ and another function $V\colon\mathbf{Diag}\to\mathcal{N}$ is $r$-differentiable at $B(x)$, then $V\circ B\colon\mathcal{M}\to\mathcal{N}$ is $\mathcal{C}^r$ at $x$ as a map between smooth manifolds.

In what follows, we consider the projections for $i,j\in\{1,\dots, c\}$,
\begin{equation*}
\pi_{i,j}\colon \mathbb{R}^{cn}\to
\mathbb{R}^{n}\times \mathbb{R}^{n},
\qquad
\pi_{i,j}(p_1,\dots,p_c) =(p_i, p_j).
\end{equation*}

\begin{proposition}\label{prop:diff_with_arbitrary_distances}
    Let $d\colon\mathbb{R}^{n}\times \mathbb{R}^{n}\to \mathbb{R}$ be a dissimilarity which is $\mathcal{C}^r$ on an open set $U\subseteq\mathbb{R}^n\times\mathbb{R}^n$, where $r\ge 0$. Let $p=(p_1,\dots,p_c)\in\mathbb{R}^{cn}$ such that $p\in \pi^{-1}_{i,j}(U)$ for all $i,j\in\{1,\hdots,c\}$. Suppose that
    $d(p_i,p_j)\neq d(p_k, p_l)$ when $\{i,j\}\neq\{k,l\}$, where $i,j,k,l\in\{1,\hdots,c\}$.
    Then the function $B_k=\textnormal{Dgm}_k\circ F$ defined in~{\rm \eqref{eqn:generating_persistence_diagram}} is $r$-differentiable at~$p$.
\end{proposition}
\begin{proof}
Since     $d(p_i,p_j)\neq d(p_k, p_l)$ for $\{i,j\}\neq\{k,l\}$, the values $d(p_i,p_j)$ for $i\neq j$ are strictly ordered. As the projections $\pi_{i,j}$ are $\mathcal{C}^\infty$ and $d$ is $\mathcal{C}^r$ on $U$, and $\pi_{i,j}(p)\in U$, we infer that $d\circ\pi_{i,j}$ is $\mathcal{C}^r$ in $p$. 
Since, in particular, $d\circ\pi_{i,j}$ is continuous, there is a neighbourhood $U'$ of $p$ where the order of the values $d(p'_i,p'_j)$ remains the same. 
Then $F(p)$ and $F(p')$ induce the same preorder on the set of simplices of $K$ for every $p'\in U\cap \bigcap_{i,j}\pi_{i,j}^{-1}(U)$.
Hence, the hypotheses of \cite[Theorem~4.7]{Leygonie2022} hold and therefore $B_k$ is $r$-differentiable at~$p$.
\end{proof}

In what follows, we denote, for a given dissimilarity $d$,
\[
\mathfrak{D}_{c,n} =\left\{(p_1,\hdots, p_c)\in\mathcal{D}_{c,n}: \text{$d(p_i, p_j)\neq d(p_k, p_l)$ if $\{i,j\}\neq\{k,l\}$}\right\}.
\]
We note that, if the dissimilarity $d(x,y)=1-|{\rm corr}(x,y)|$ is chosen, then $\mathfrak{D}_{c,n}$ is an open dense subset of $\mathbb{R}^{cn}$, since $d(p_i,p_j)=d(p_k,p_l)$ precisely when $|{\rm corr}(p_i,p_j)|=|{\rm corr}(p_k,p_l)|$, and the square of correlation is a rational function.

\begin{proposition}
\label{prop:Proposition2}
  Let $c,n\geq 2$, $k\in\mathbb{Z}_{\geq 0}$, and $d\colon\mathbb R^n\times\mathbb R^n\to[0,1]$ be the dissimilarity given by $d(x,y)=1-|{\rm corr}(x, y)|$. The function $B_k={\rm Dgm}_k\circ F$ defined in~{\rm \eqref{eqn:generating_persistence_diagram}} is $\infty$-differentiable on $\mathfrak D_{c,n}$.
\end{proposition}

\begin{proof}
  Take $p=(p_1,\hdots, p_c)\in\mathfrak{D}_{c,n}$. By the definition of $\mathfrak D_{c,n}$, we have that $d(p_i,p_j)\neq d(p_k, p_l)$ for all $\{i,j\}\neq\{k,l\}$.
  Furthermore, the correlation is well-defined and $\mathcal{C}^\infty$ on $\pi_{i,j}(p)=(p_i, p_j)$ for all $i,j\in\{1,\hdots, c\}$, because $\text{cov}(p_i, p_j)\ne 0$ for points in $\mathfrak{D}_{c,n}$. We also have that $|\text{corr}(x,y)|$ is $\mathcal{C}^\infty$ on every $(p_i,p_j)$, since the absolute value function is $\mathcal{C}^\infty$ on $\mathbb R\smallsetminus\{0\}$. Therefore, $d$ is $\mathcal{C}^\infty$ on $\pi_{i,j}(p)=(p_i, p_j)$ for all $i,j\in\{1,\hdots, c\}$ and thus the assumptions of Proposition~\ref{prop:diff_with_arbitrary_distances} hold, implying that $B_k$ is $\infty$-differentiable on $\mathfrak D_{c,n}$.
\end{proof}

\begin{proof}[Proof of Theorem ~{\rm\ref{thm:differentiability_regularisation_terms}}]
By
Proposition~\ref{prop:Proposition2}, the function $B_0$ is $\infty$-differentiable on $\mathfrak D_{c,n}$, where $B_0(X)$ is the persistence diagram of $X$ in homological dimension zero. Therefore, we only need to display functions $\tilde{\mathcal{T}}_i\colon\mathbf{Diag}\to\mathbb R$ that are $\infty$-differentiable on $B_0(\mathfrak D_{c,n})$ such that $\mathcal{T}_i = \tilde{\mathcal{T}}_i\circ B_0$ for $i\in\{1,2\}$.

Here we view zero-dimensional persistence diagrams as consisting of points $(0,y)$, although we keep denoting them in the general form $(b,d)$.
When computing the average persistence and standard deviation of persistence of the points in $B_0(X)$, the number of points in the diagram is assumed to be equal to the number of edges of a minimum spanning tree for $(X,d)$, that is, $|X|-1$.
Therefore, the functions $\tilde{\mathcal{T}}_i$ can be defined as
\begin{equation*} \tilde{\mathcal{T}}_1(D) 
\triangleq 
\frac{1}{c-1}\sum_{(b,d)\in D^*}\!\!(d-b),
\qquad
\tilde{\mathcal{T}}_2(D) \triangleq -\alpha \tilde{\mathcal{T}}_1(D) + \beta\sigma(D),
\end{equation*}   
where 
\begin{equation*}
\sigma^2(D)\triangleq\frac{1}{c-1}\sum_{(b,d)\in D^*}((d-b)-\tilde{\mathcal{T}}_1(D))^2
\end{equation*}
with $D^*=\left\{(b,d)\in D:  d<\infty\right\}$. Points in the diagonal $\Delta^\infty$ are sent to zero by $d-b$ and are not taken into consideration in the sums, and neither are points at infinity.

In order to prove that the functions $\tilde{\mathcal{T}}_i$ are $\infty$-differentiable, take any $m,\ell\in\mathbb Z_{\geq 0}$ and $\tilde{D}\in\mathbb R^{2m}\times\mathbb R^\ell$ such that $Q_{m,\ell}(\tilde{D})=D$. 
If we write $\tilde{D}=(x_1,y_1,\hdots, x_m, y_m, z_1, \hdots, z_\ell)$, then the functions $\tilde{\mathcal{T}}_i\circ Q_{m,\ell}$ are given by
\begin{equation*}
\tilde{\mathcal{T}}_1( Q_{m,\ell}(\tilde{D})) = 
\frac{1}{c-1}\sum_{i=1}^m(y_i-x_i),\qquad
\tilde{\mathcal{T}}_2( Q_{m,\ell}(\tilde{D})) = -\alpha\tilde{\mathcal{T}}_1(Q_{m,\ell}(\tilde{D})) + \beta\sigma(\tilde{D}),
\end{equation*}
where 
\begin{equation*}
\sigma^2(\tilde{D})\triangleq\frac{1}{c-1}\sum_{i=1}^m[(y_i-x_i)-\tilde{\mathcal{T}}_1(Q_{m,\ell}(\tilde{D}))]^2.
\end{equation*}
      
The functions $\mathcal{\tilde{T}}_i\circ Q_{m,\ell}$ are $\mathcal{C}^\infty$ on all their domain because they are compositions of $\mathcal{C}^\infty$ functions on a neighborhood of $\tilde{D}$. The only function that is not $\mathcal{C}^\infty$ in all its domain is the square root function, which is not differentiable at zero. However, for points $p\in\mathfrak D_{c,n}$ we have pairwise different distances, and consequently the persistence diagram $B_0(X)$ contains at least two different points, making $\sigma(\tilde{D})\neq 0$ for a neighbourhood $U$ of $\tilde{D}$ and thus making $\sigma(\tilde{D})$ a $\mathcal{C}^\infty$ function on~$U$. Hence, $\tilde{\mathcal T}_i$ is $\infty$-differentiable. Therefore, as $B_0$ and $\tilde{\mathcal{T}}_i$ are $\infty$-differentiable in $\mathfrak D$ and $B_0(\mathfrak D_{c,n})$ respectively, 
the functions $\mathcal{T}_i$ are $\mathcal{C}^\infty$ on $\mathfrak D_{c,n}$.
\end{proof}

\end{document}